\documentclass{article}

\usepackage{arxiv}

\usepackage[utf8]{inputenc} 
\usepackage[T1]{fontenc}    
\usepackage{hyperref}       
\usepackage{url}            
\usepackage{booktabs}       
\usepackage{amsfonts}       
\usepackage{nicefrac}       
\usepackage{microtype}      
\usepackage{lipsum}		
\usepackage{graphicx}
\usepackage{natbib}

\usepackage{doi}
\usepackage{amsmath}
\usepackage{multirow}
\usepackage[english]{babel}
\usepackage{amsthm}
\newcommand{\beginsupplement}{%
        \setcounter{table}{0}
        \renewcommand{\thetable}{S\arabic{table}}%
        \setcounter{figure}{0}
        \renewcommand{\thefigure}{S\arabic{figure}}%
     }
\setcitestyle{authoryear, open={(},close={)}}

\title{Why Deep Learning Generalizes}

\author{ \href{https://orcid.org/0000-0003-1661-4579}{\includegraphics[scale=0.06]{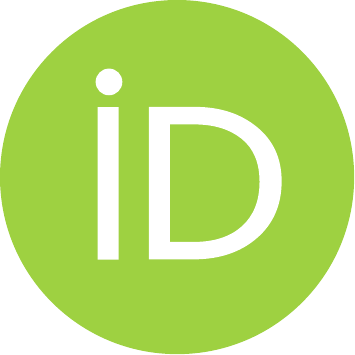}\hspace{1mm}Benjamin L. Badger}\thanks{The author would like to thank Guidehouse for support during the research and writing of this paper.  Code for this work may be found on \url{https://github.com/blbadger/learning-generality}} \\
	Guidehouse \\
	1200 19th St. NW Washington, DC 20036 \\
	\texttt{bbadger@guidehouse.com} \\
}

\date{}

\renewcommand{\headeright}{Technical Report}
\renewcommand{\undertitle}{Technical Report}
\renewcommand{\shorttitle}{\textit{arXiv} Template}

\hypersetup{
pdftitle={Why Deep Learning Generalizes},
pdfsubject={Deep Learning},
pdfauthor={Benjamin L. Badger},
pdfkeywords={Deep Learning, Gradient Descent, Generalization, Overfitting},
}

\begin{document}
\maketitle

\begin{abstract}
    Very large deep learning models trained using gradient descent are remarkably resistant to memorization given their huge capacity, but are at the same time capable of fitting large datasets of pure noise.  Here methods are introduced by which models may be trained to memorize datasets that normally are generalized.  We find that memorization is difficult relative to generalization, but that adding noise makes memorization easier.  Increasing the dataset size exaggerates the characteristics of that dataset: model access to more training samples makes overfitting easier for random data, but somewhat harder for natural images.  The bias of deep learning towards generalization is explored theoretically, and we show that generalization results from a model's parameters being attracted to points of maximal stability with respect to that model's inputs during gradient descent.
\end{abstract}


\section{Introduction}

    For many years it was assumed that severely over-parametrized deep learning models would fail to generalize due to overfitting.  This notion was an inspiration for many advances in the field: for example, the convolutional neural network was originally proposed to allow for generalization by applying prior knowledge about natural images, specifically that the first few layers should contain local connections only and the model should be translation-invariant \citep{LeCun1989}.  Convolutional models were thus noted to present a contrast to the fully connected neural network architecture that was prevalent at the time, an architecture which had been shown to be capable of approximating any arbitrary computable function given sufficient model size \citep{hornik1989multilayer}.  As it is extremely unlikely that a randomly chosen function exactly fitting training data generalizes to test data without prior knowledge, it is perfectly reasonable to assume that over-parametrized deep learning models will not generalize.
    
    Subsequently it has come as some surprise that deep learning models do not exhibit the expected behavior with respect to generalization: increasing a model's size (for convolutional or fully connected models alike) far beyond that required to memorize a dataset typically does not lead to overfitting \citep{neyshabur2014} even though large convolutional models are found to be capable of fitting noise with apparent ease \citep{zhang2016}.  This is true even when all extrinsic regularization techniques are abandoned \citep{neyshabur2014, zhang2016}.  These observations have led to the hypothesis that it is not model capacity as it is traditionally understood but some other regularizer that prevents deep learning models from overfitting even when they are capable of doing so \citep{Zhang2021}.
    
    One logical place to look for an model-extrinsic implicit regularizer is in the learning method applied to that model, and therefore role of gradient descent in deep learning regularization is the focus of this work. 
    
    The role of stochastic gradient descent as an implicit regularizer has been proposed to be a function of the ratio of the learning rate to batch size \citep{Smith2021}, although elsewhere it has been observed that regularization does not depend on forming minibatches during gradient descent \citep{Wu2017}.  Elsewhere it has been observed that stochastic gradient descent converges to regions of good generalization (in particular regions containing small gradients of the output with respect to the model's parameters) \citep{Lei2018} for natural language tasks. In related work, the role of gradient descent in realistic scenarios in which parameters are changed via discrete updates has also been investigated \citep{barrett2020implicit}, with the conclusion that discrete updates prevent parameter trajectories from entering regions that fail to generalize and instead bias towards model configurations in which the loss is relatively insensitive to parameter perturbation.  It remains unclear why convergence would occur in regions of generalization even if the learning rate is set to be very small or else becomes very small during training (as typically occurs for adaptive gradient descent-based optimization algorithms used for this work), and why convergence to regions of generalization should occur at all.
    
    Like many models with intractable objective functions, neural networks are usually trained using gradient descent. Gradient-based learning applied to neural networks was for a time viewed with suspicion for its lack of guarantees: given a non-convex objective function typical of nonlinear models like neural networks, gradient descent may or may not converge on a sufficiently small objective value \citep{goodfellow2016deep}. This is because locally optimal solutions are not guaranteed to be globally optimal, but gradient descent only makes locally optimal gradient updates which may lead to the approximation of some sub-optimal critical point.
    
    It should first be noted that for very large models, however, this perspective may be inverted: the inability of gradient descent to reach a gobal minimum in the objective function is one reason why these models tend to resist overfitting, as an over-parametrized model that has reached gobal objective minima is likely to overfit for common classification tasks given sufficient model size.
    
    The more applicable case of overfitting when an objective function global minimum is not reached is considered. This is the case that is normally of greatest interest to deep learning as gradient descent does appear capable of finding configurations that overfit for certain datasets, and is explored here experimentally and theoretically.
    
    Our primary findings may be summarized as follows: for sufficiently large models applied to typically generalizable datasets,
    
    \begin{center}
	\it{Model configurations that severely overfit are difficult to find via gradient descent}
    \end{center}
    
    and it can be shown that  
    
    \begin{center}
	\it{Gradient descent is globally optimal with respect to finding generalizable model configurations}
    \end{center}
    
    Implications of these ideas are discussed, and methods are introduced to measure the tendency of a given dataset in resisting overfitting.
    
\section{Overfitting is Empirically Difficult for Gradient Descent}

    Surprising observations have been made with regards to large deep learning model generalization: first that models are capable of fitting pure noise with little apparent difficulty \citep{arpit2017closer, zhang2016}, that increasing a vision model's size far beyond what is necessary to memorize does not actually induce memorization \citep{neyshabur2014}, and finally that regularization applied extrinsically is not necessary to prevent overfitting for most models \citep{neyshabur2017}.  We focus on the first observation to begin.
    
    A model's hypothesis space is defined as the set of functions that it may approximate, where the learning process is applied in order to select one or some subset of these functions. Here is introduced the notion of a constrained hypothesis space, defined as the set of functions a model can approximate subject to some learning procedure such that the learning procedure succeeds.  One may think of these functions as being the set of model configurations $\theta$ that achieve a sufficiently small objective function value. Changing a model's training dataset typically changes its constrained hypothesis space for some given objective function, assuming that the learning procedure is successful. 
    
    With a dataset that sufficiently limits the constrained hypothesis space, it can be shown that fitting noise can actually occur in fewer epochs than general learning on different datasets \citep{badger2021}.  But these results do not provide evidence for the idea that overfitting is relatively easy in datasets that normally do not lead to significant overfitting, as the constrained hypothesis space is generally different for different datasets such that the task of approximating a member of this space (and thus training difficulty) is generally different between datasets.
    
    Therefore it remains to be shown that a model which typically generalizes on some given dataset is also capable of overfitting that same dataset using gradient descent.  For this work overfitting is defined as the difference between training accuracy and test accuracy, where severe overfitting (also known as memorization) occurs when the training set accuracy approaches unity while the test set accuracy remains near a naive threshold (ie 1/10 for a 10-class set of equal proportion).  
    
    To investigate a model's ability to severely overfit datasets that normally generalize, an atypical objective function denoted in Equation (\ref{eq1}) was constructed where $L(O(a, \theta), y)$ signifies some chosen loss function applied to the output of model $\theta$ given label $y$ and input $a$. 

    \begin{equation}
        J'(O(a, \theta)) = L(O(a_{train}, \theta), y_{train}) - L(O(a_{test}, \theta), y_{test})
        \label{eq1}
    \end{equation}
    
     $J'$ as described in (\ref{eq1}) is an unstable objective function: once symmetry is broken and the right term is smaller than the left, simply increasing the value of all output elements in $O(a, \theta)$ leads to a further decrease in $J'$ but does not increase the difference between training and test set model accuracy.  The model can be regularized to prevent this, but we took the somewhat more direct approach of simply transforming the output such the sum of all output elements is constant.  This is done by passing the output into a Softmax function which such that the model output is described by Equation (\ref{eq1-1}) where $O'$ signifies the raw, untransformed outputs from the model in question.  
    
    \begin{equation}
        O(a, \theta) = \mathbf{softmax} \; O'(a, \theta)
        \label{eq1-1}
    \end{equation} 
    
    The loss function in (\ref{eq1}) is cross-entropy, or more specifically a log Softmax transformation followed by a negative log loss as is standard in recent versions of PyTorch \citep{paszke2019pytorch} based on the original the Torch library \citep{collobert2002torch}. To make the comparisons in this work consistent, we add Softmax layers to all models in this work regardless of whether or not they are applied to the objective function (\ref{eq1}), with the sole exception being the model used in Figure \ref{figs1} (a).
    
    Minimizing the split gradient in Equation (\ref{eq1}) is equivalent to minimizing some objective function $J$ for the model on training data while maximizing that same function on test data. Although the selection of an objective function typically changes the constrained hypothesis space of a given model, observe that any model parameter configuration $\theta$ that minimizes (\ref{eq1}) also minimizes the standard training objective function (\ref{eq2}) assuming that the minimization of (\ref{eq1}) minimizes both left and right terms.
    
    \begin{equation}
        J(O(a, \theta)) = L(O(a_{train}, \theta), y_{train}) 
        \label{eq2}
    \end{equation}
    
    Therefore given some fixed dataset the constrained hypothesis space of (\ref{eq1}) is a proper subset of the constrained hypothesis space of (\ref{eq2}) such that any $\theta$ that minimizes (\ref{eq1}) will minimize (\ref{eq2}), but not every $\theta$ that minimizes (\ref{eq2}) will also minimize (\ref{eq1}). 

    To investigate the ability of severe overfitting within a single constrained hypothesis space, the first 10k examples from CIFAR10 \citep{krizhevsky2009learning} test and training datasets were employed. We apply (\ref{eq1}) to a relatively small 5-layer convolutional network (905,930 parameters with the last layer being fully connected, for more information see the source code) with the adaptive learning algorithm Adam \citep{kingma2014adam} as the optimization technique.  When this is done it is apparent that severe overfitting occurs, in this case with the training accuracy above 90\% and the test accuracy of 2.4\% (Figure \ref{fig1}). Compare this to when the standard gradient (\ref{eq2}) is employed for the same model and dataset, where 90\% training accuracy corresponds to approximately 60\% test accuracy.   
    
    It may be wondered if memorization can be induced without knowledge of a test dataset. This can be tested by observing the ability of gradient splitting in (\ref{eq1}) to induce overfitting by dividing a training dataset into two equal portions and arbitrarily assigning one of these to be the `test' set, and then observing the classification accuracy on the real test set (or on another held-out set of training data). As shown in Figure \ref{fig1}, gradient splitting indeed is capable of inducing overfitting with respect to an unseen dataset 
    
    It is also apparent that many epochs are necessary for overfitting compared to general training.  The number of training steps necessary to reach some threshold, denoted here in epochs, is used as a measure of the `difficulty' of training.  Therefore gradient-based learning is capable of obtaining a model configuration $\theta$ that exhibits severe overfitting in a constrained hypothesis space in which such overfitting normally is exceedingly rare. It is clear from Figure \ref{fig1} that severe overfitting via gradient descent on large, approximately i.i.d. datasets such as CIFAR10 is empirically difficult.  
    
    \begin{figure}[h]
        \centering
        \includegraphics[width=0.85\textwidth]{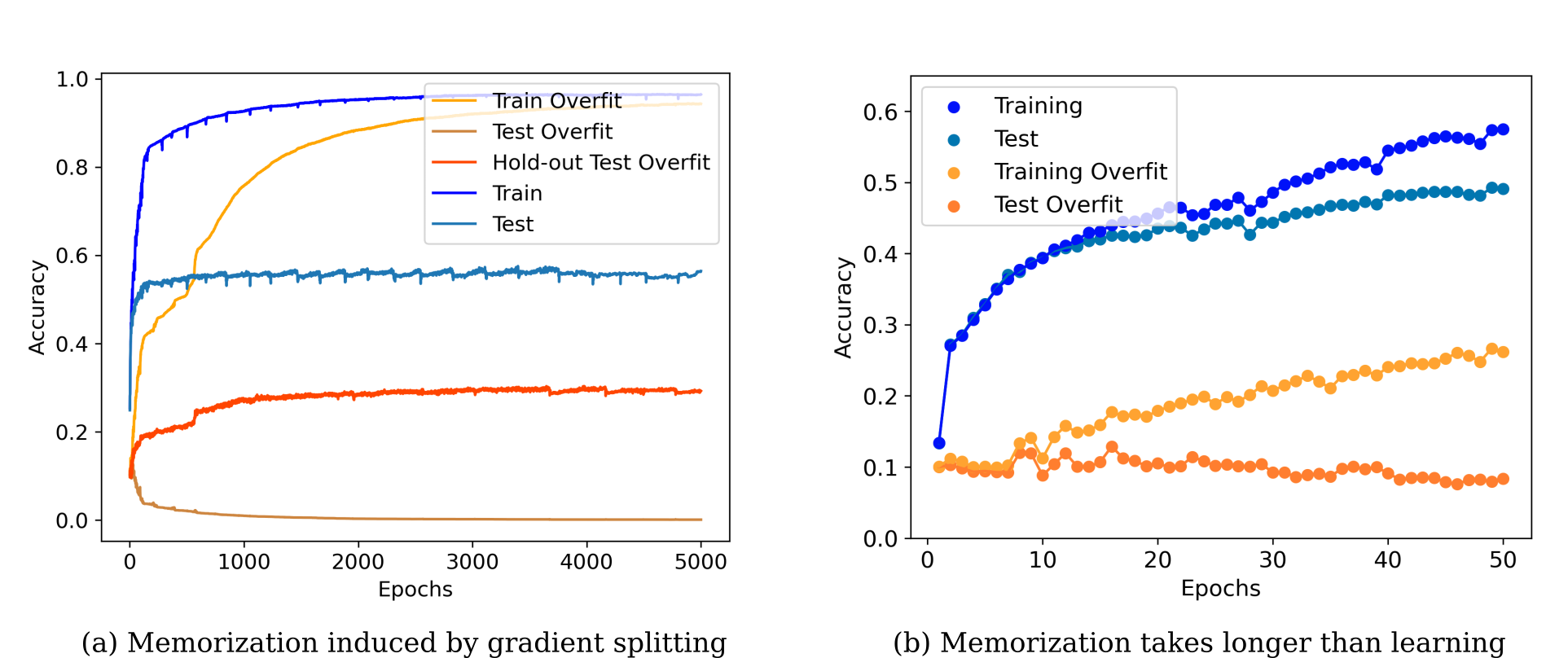}
        \caption{Gradient splitting induces overfitting, but learning to overfit is more difficult than learning to generalizable on CIFAR10 subsets. (a) Gradient splitting with Equation (\ref{eq1}) induces severe overfitting for the first 10k images, lr=2e-4.  (b) Memorization requires many more epochs than generalizable training.}
        \label{fig1}
    \end{figure}
    
    The overfitting induced by Equation (\ref{eq1}) is in a sense artificially severe, and actually achieves anti-learning as the test classification accuracy heads to 0.  A more realistic case of memorization is one in which the test set classification accuracy simply does not increase as training accuracy increases. To achieve this, a similar loss to (\ref{eq1}) is constructed in Equation (\ref{eq4}), where $m(a_{test}, y_{test})$ is a measure of the accuracy of the model on the test set and $k$ is the accuracy threshold desired.
        
    \begin{equation}
      J^0(O(a, \theta)) = 
      \begin{cases}
        L(O(a_{train}, \theta), y_{train}) - L(O(a_{test}, \theta), y_{test}) & \text{if } m(a_{test}, y_{test}) > k \\
        L(O(a_{train}, \theta), y_{train}) & \text{else} \\
      \end{cases}
      \label{eq4}
    \end{equation}
    
    With the assumption that (\ref{eq1}) sends the test set accuracy to 0 more rapidly than the training set accuracy to 1 (which is empirically found to be the case for this work), (\ref{eq4}) is guaranteed to converge on a test set error of $k$ if it is sufficiently small.  Employing this objective function, it can be shown that it is still relatively difficult to severely overfit CIFAR10 (Figure \ref{fig2}).  On the other hand, overfitting pure noise is relatively easy (corroborating previous work \citep{zhang2016}) and indeed the incremental addition of noise decreases the number of epochs required to overfit. 
    
    Next we investigated the relationship of dataset size to the ease of memorization with (\ref{eq4}).  As shown in Figure \ref{fig2} (c), increasing the size of the dataset has opposite effects on the number of epochs required to overfit noise versus CIFAR images: overfitting noise becomes easier, but overfitting natural images is found to be more difficult as the training dataset size increased from 2.5k to 10k samples. 
    
    \begin{figure}[h]
        \centering
        \includegraphics[width=0.9\textwidth]{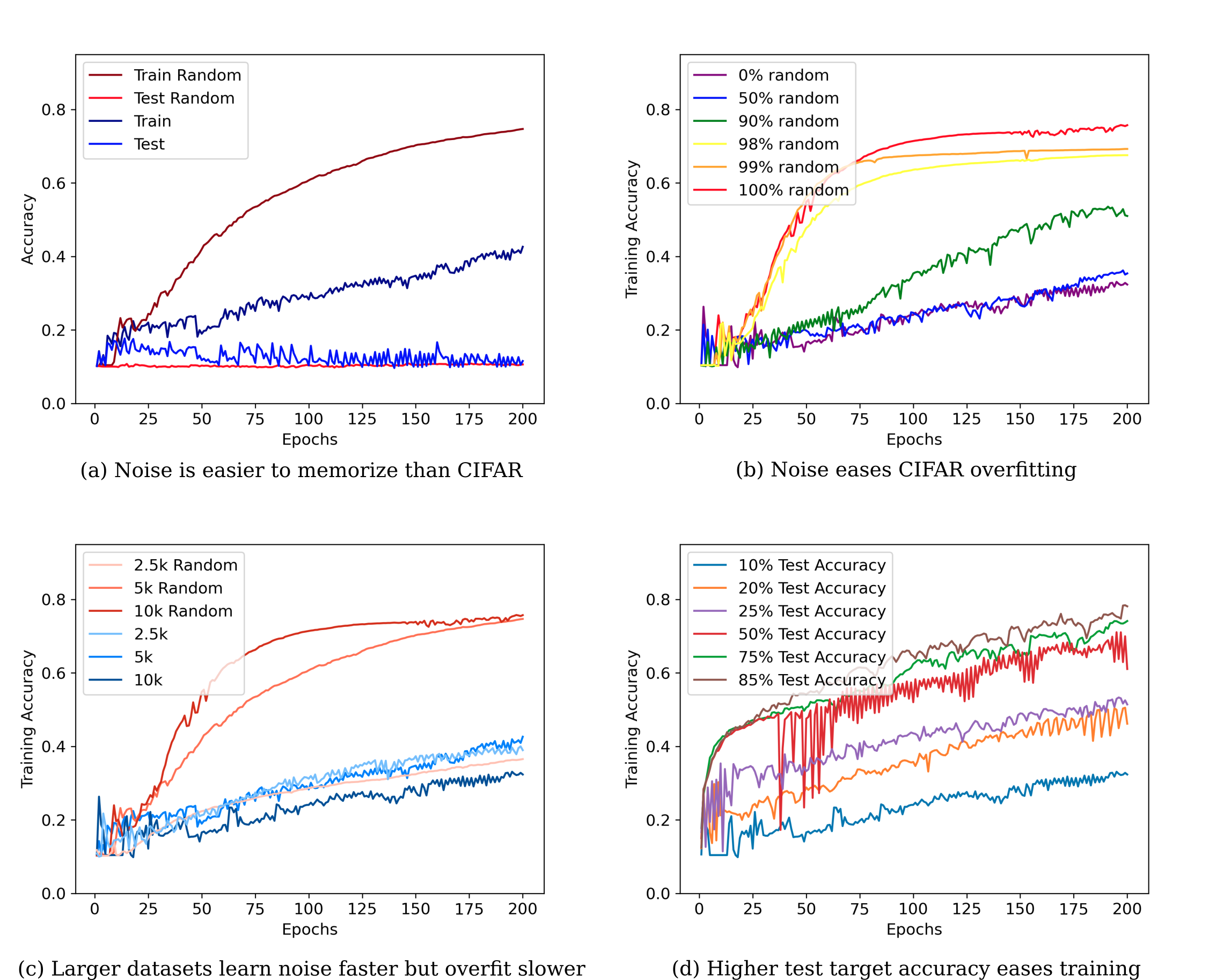}
        \caption{Memorizing noise is easier than memorizing CIFAR10 applying (\ref{eq4}), unless otherwise noted.  Conditionals are evaluated at the start of each epoch. (a) Learning curves for 5k CIFAR10 versus fixed scaled Gaussian noise given $k=0.11$. (b) Learning curves for 10k CIFAR10 mixed with fixed scaled Gaussian noise. (c) Learning curves for datasets of the denoted size (ie 2.5k signifies 2560 images in both training and test datasets) for CIFAR10 versus scaled Gaussian noise.  (d) Learning curves for (\ref{eq5}) with upper and lower limits equal to the target denoted. }
        \label{fig2}
    \end{figure}
    
    The same finding that dataset size has opposite effects on overfitting ease for noise versus real image data is also observed for ResNet18 (Supplementary Figure \ref{figs1}), suggesting that this effect is not limited to a specific model choice.  It is interesting to note that the addition of a Softmax transformation to the output of ResNet18 appears to make memorization of random data rather slow for all dataset sizes, necessitating the removal of this layer for experiments in Figure \ref{figs1} (a).  It is unclear why a Softmax transformation would affect the ability of noise memorization in ResNet18 but not the smaller convolutional model used elsewhere in this work.
    
    With a related construct we introduce the objective function $J^t$ such that the test set error is not sent to the origin or to the naive threshold but instead is sent towards some value via gradient descent. As shown in Equation (\ref{eq5}), the desired lower bound is indicated by $l$ and the desired upper bound on the training accuracy is denoted $u$, where $l \leq u$. Here for simplicity we consider $J^t$ for the case where $u=l$ and find that larger the target test accuracy $u=l$ is, the easier training on CIFAR10 becomes (Figure \ref{fig2} (d)).
    
    \begin{equation}
      J^t(O(a, \theta)) = 
      \begin{cases}
        L(O(a_{train}, \theta), y_{train}) - L(O(a_{test}, \theta), y_{test}) & \text{if } m(a_{test}, y_{test}) > u \\
        L(O(a_{train}, \theta), y_{train}) + L(O(a_{test}, \theta), y_{test}) & \text{if } m(a_{test}, y_{test}) < l \\
      \end{cases}
      \label{eq5}
    \end{equation}
    
    In conclusion, it is observed that a dataset that is normally generalized by a model is capable of being severely overfit by that same model, but that the parameters necessary for memorization are very difficult for gradient descent to find.  We now turn to theoretical work on why this would be the case.

\section{Model configurations corresponding to global objective function minima are not desirable but typically are not found via gradient descent}
    
    For a non-convex loss function, gradient descent typically does not make globally optimal updates to a model's parameters with respect to finding parameters of minimum loss.  In algorithmic terms, gradient descent is a greedy algorithm because it minimizes a cost function for a model's current configuration whether or not the subsequent configuration is any closer to the global minimum for that loss function.  Poor correspondence between local and global structure occurs by definition for non-convex objective functions, such that for some initial value $\theta_0$ there is no guarantee that subsequent configurations $\theta_1, \theta_2, ...$ will approximate global rather than local minima or non-minimal points.
    
    Much work on gradient descent applied to non-convex objective functions is focused on the ability of gradient descent to find some sufficiently small value given the limitations described in the last section.  For small models with fewer parameters than unique training examples, the global minimum of an objective function is a desirable value because it typically maps to a model configuration that has low training error subject to the model's intrinsic bias. For large models the situation is reversed: obtaining a global minimum is not desirable because models of sufficient capacity are capable of nearly perfectly approximating the training data distribution, and each (typically an infinite number due to hidden layer non-identifiability) model configuration mapping to a global minimal objective value has done so. 
    
    It has been observed that overcomplete linear models are capable of avoiding overfitting, but only if certain priors about the data are assumed \citep{zhang2016}.  Parameter norming and output interpolation are two priors that are used in various overcomplete linear models to prevent overfitting.  There are no such priors on typical deep learning models, although (minimum) parameter norming is typically enforced weakly through initialization methods that assign model weights and biases to start near the origin before training if early stopping is implemented.

    We give an intuitive explanation for why global minimization of $J$ is not desirable for the case of input classification.  Deep learning models applied to problems of classification are typically trained with a cost (objective) function of maximum likelihood, perhaps the most common formulation of which being cross-entropy. Maximum likelihood estimation in the context of classification is equivalent to minimization of cross-entropy, where the objective function $J$ to be minimized is given in (\ref{eq6}).
    
    \begin{equation}
        \begin{gathered}
        J(O(a, \theta), y) = - \mathbb{E}_{x \sim p_{data}} \; \log p_{model}(x) \\
        J(O(a, \theta), y) = - \sum_i y_i * \log (O_i(a, \theta))
        \end{gathered}
    \label{eq6}
    \end{equation}
    
    It is clear that minimization of (\ref{eq6}) yields an output vector $y'$ identical to the target $y$ in which case $J(O(a, \theta), y) = 0$. Observe that in the process of generalizable learning we desire some model $\theta$ that yields output $y'$ which necessarily reflects some information about the classification process itself.  But now consider what the model implies by minimizing (\ref{eq6}) for a typical case where $y$ is one-hot encoded: each individual sample is `as far' from every other not of that category with respect to the output.  With real images, however, the generalizable learning process may be thought of as distinguishing not just the correct classification of training inputs but also distinguishing how `similar' each output should be from each other, such that unseen inputs (that are not members of the set of training inputs) may be assigned accordingly.
    
    Next we turn to the case where overfitting occurs in spite of a lack of global objective function minimization, which is more applicable for gradient descent-based deep models.

\section{Why Gradient Descent tends to Generalize}

    Experiments presented here and elsewhere provide evidence for the idea that deep learning models trained via gradient descent are biased towards generalization, and tend to severely overfit only when generalizable learning is impossible or extremely unlikely.  This begs the question: why does gradient descent tend to generalize?

    By comparing the norm of the gradient of the model's parameters $\theta$ while training for generalization versus training on a randomized dataset, it is found that the maximum value of the norm of the model's gradient $||\nabla_{\theta} J(O(a, \theta))||$ is greater for the generalizable $J$ than $J$ which overfits.  But this observation does not provide an explanation for why $\theta$ would head towards generalizable configurations at the start of training, when that gradient is usually small for both generalizable and non-generalizable $J$.

    The question of why deep learning generalizes can be reframed as a question of why parameter configurations $\theta$ that correspond to overfitting are weaker attractors than parameter configurations that correspond to generalizable learning.  To understand this, we can turn to the theory of dynamical equations.
    
    Gradient descent as a dynamical system is expressed as Equation (\ref{eq10}). 
    
    \begin{equation}
    \theta_{n+1} = \theta_{n} - \epsilon * \nabla_{\theta_n} J(O(a, \theta_n))
    \label{eq10}
    \end{equation}
    
    This system is nonlinear as long as the function describing the model $O(a, \theta)$ has at least one hidden layer, regardless of whether nonlinear activations are applied to each layer. Therefore we must apply nonlinear dynamical theory to understand the properties of (\ref{eq10}). 

    \subsection{Small Changes in High-Dimensional Attractor Stability lead to Large Changes in Attractor Basin Volume}

    An attractor in dynamical system theory is some fixed point, in this case $\theta_m$, for which certain initial values of $\theta$ head towards. Nonlinear systems may have critical (i.e. fixed) points where $\theta_{n+1} = \theta_n$, in which case $\nabla_{\theta_n} J(O(a, \theta_n) = 0$, that are not attractors.  All points in some region that head towards an attractive point (which for gradient descent is a minimal point) are termed the `basin of attraction', and the size of the basin of attraction depends on certain characteristics of the attractive point. This means that deep learning models trained via gradient descent may have many minimal points (with respect to the objective function) that will not be arrived at, and other minimal points that they will.
    
    \begin{figure}[h]
        \centering
        \includegraphics[width=0.65\textwidth]{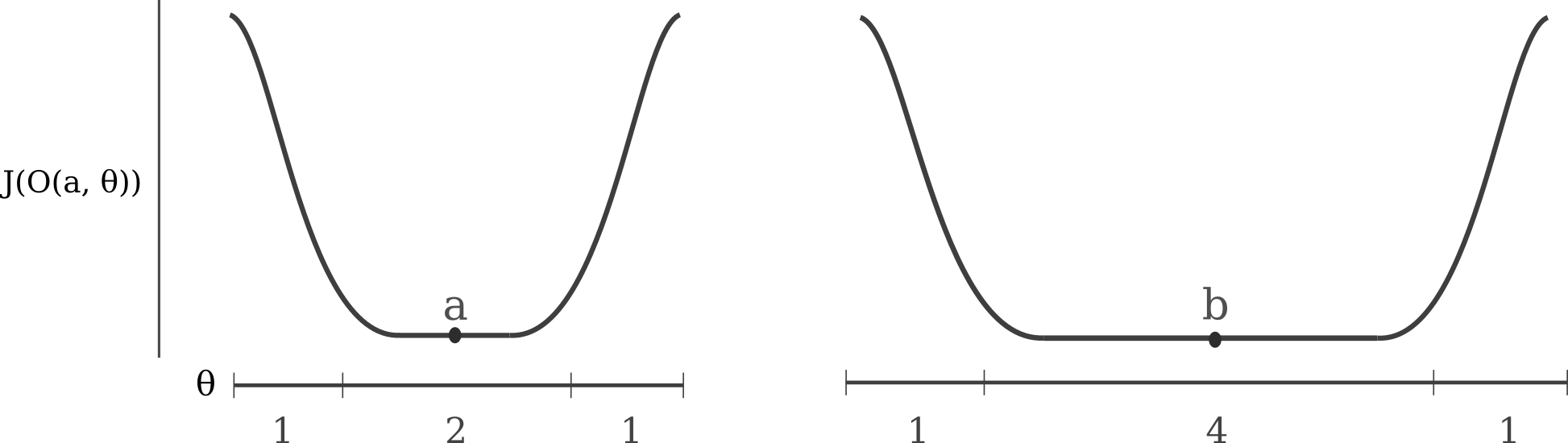}
        \caption{Simplified attractor geometry for the loss function value (vertical axis) for model of one parameter (horizontal axis)}
        \label{fig3}
    \end{figure}
    
    Define the stability of an attractive point to be the minimal amount $\epsilon$ that can be added in any direction to the attractor $\theta_m$ such that the output is within some given value $\delta$. To gain some appreciation for how a small change in stability can lead to a large increase in total attractor basin size, consider two minimal points denoted $a$ and $b$ in Figure \ref{fig3}.  For simplicity, assume that the slope towards the attractive point is the same from any direction and the total measure of the space of that slope is $m + m = 2m$ in one dimension.  Now observe that $b$ is twice as stable with respect to the output objective $J(O(a, \theta))$ than $a$ as measured by the distance needed to perturb a point at the attractor in order to make some small change $\delta$ in that output.  With these assumptions, we can clearly see that in one dimension the ratio of the size of attractor $a$ to $b$ is $4/6 = 2/3$. The same ratio applies to the total size of the stable region in attractor $b$, denoted $s_b$, to the total attractor basin area $s_b + 2m$, ie $s_b / (s_b + 2m) = 2/3$. 
    
    Deep learning models are equivalent to transformations in high-dimensional vector space, so we can gain intuition into the space describing such models by generalizing the case shown in Figure \ref{fig3} to many dimensions while observing the ratios of the sizes of the two attractors $a/b$.  It is clear in this case that in $n$ -dimensional space, the ratio of the size of $a$ to $b$ is $2^n / 3^n$ such that the basin of attraction for $a$ decreases exponentially compared to the basin size for $b$.  To illustrate just how quickly this growth occurs, observe that with a 1000-dimensional space (corresponding to a very small model) the ratio $a/b = (2/3)^{1000} = 1/1.2 * 10^{176}$ such that it is virtually impossible to find basin $a$ with a point that is initialized at random.  
    
    If the basin of attraction of $b$ is so large, it may be wondered how random initialization is not observed to exhibit a model configuration in the $\delta$-stable region $s_b$.  This is because the ratio of the $\delta$-stable region to the entire basin of attraction is $s_b/(s_b + 2m) = a/b = 1/1.2 * 10^{176}$, meaning that there is an astronomically small chance of picking a point at random that exists in $s_b$. Intuitively, this is a result of the fact that almost all the `volume' of a high-dimensional sphere is located near its edges.
    
    It may be wondered if this idealized viewpoint is at all representative of gradient descent applied to deep learning models. There is indeed evidence that it is: such models are normally observed to begin the training procedure in a region of high cost before then experiencing a sharp decline, which is followed by epochs in which the gradient norm decreases.  As an example, the Frobenius norm of the first weight layer of the convolutional model used in this paper $||\nabla_{\theta_1} J(O(a, \theta)) ||_2$ applied to CIFAR10 peaks around epoch 50 at just under 1, before decreasing steadily to around 5e-6. When this is compared to the same norm for the same model trained using (\ref{eq1}), we see that there is a smaller peak (0.1) and a larger end value (5e-3). The tendency of gradient descent applied to deep learning models to end up in regions of small gradient norm has been observed elsewhere \citep{barrett2020implicit, Lei2018}. Note, however, that the value of $||\nabla_{\theta} J(O(a, \theta)) ||_2$ is not strictly applicable to the definition of `stability' given above, the former is defined as the maximum discrete distance in parameter space one can travel before moving a certain distance in loss whereas the latter is simply the size of the maximum `slope' in multidimensional space.  It is easy to imagine a cost function landscape that has high gradient value everywhere but is stable (ie a sawtooth) and likewise an unstable landscape that has relatively small gradient norm.

    In the case above the attractor is convex, and more specifically it is a n-dimensional ball.  More generally, it is clear that small increases in attractor stability (as defined above) lead to large changes in the total `volume' of the basin of attraction regardless of the shape of that basin, and in particular regardless of whether it is convex.  This is due to the Brunn-Minkowski inequality, which relates the Lebesgue measure $\mu$ of combined nonempty compact sets $A, B \in \Bbb R^n$ to the sum of their separate measures and is given in Equation (\ref{eq20}) for reference. The combination (or extension) of $A$ and $B$ is denoted $A + B$ and is defined as the set of all elements that can be obtained by adding any one element of $A$ to one element of $B$.

    \begin{equation}
        \mu(A + B)^{1/n} \geq \mu(A)^{1/n} + \mu(B)^{1/n}
        \label{eq20}
    \end{equation}
    
    For the Brunn-Minkowski inequality to hold as the dimension number $n$ increases for some fixed sets $A, B$ it is apparent that the Lebesgue measure (informally the Euclidean `volume') of the combined set must become extremely large relative to the sum of separate volumes.  Without loss of generality, taking $A$ to be our attractor and $B$ to be a relatively small set corresponding to an extension on $A$ due to increased stability, we find that large dimensionality lead to a large increase in the volume $\mu(A + B)$. 
    
    Intuitively the reason as to why the relation between size and addition for various dimensions holds for arbitrary sets as well as for convex ones is due to isoperimetric principle of the n-dimensional ball.  We have seen that most of the `volume' of an n-dimensional ball (for large $n$) is located near its surface.  Now consider that no set has a smaller surface-to-volume ratio than the ball, and it may be unsurprising that a combination of a small amount to an arbitrary set in $n$ dimensions yields at least as large an increase in size as the addition of a small amount to a ball.

    \subsection{Stability in Parameter Space Corresponds to Stability in the Input Space}

    Local minima that are stable in $\theta$ space therefore have larger basins of attraction than points that are less stable (with stability defined as a scalar amount by which $\theta$ can be perturbed in any given direction without changing the loss value by more than some given amount).  Defining generalization as insensitivity of a model's loss function with respect to `small' changes in the input, this can be shown to imply that these models generalize.  This definition of generalization is one that assumes smoothness but no other prior on the data distribution.

    \topsep=10pt
    \newtheorem{theorem}{Theorem}

    \begin{theorem}
        Stable model configurations generalize.
    \end{theorem}

    \begin{proof}
    Assume that gradient descent can be expressed by Equation (\ref{eq10}), and that $\nabla_{\theta_n}$ is some continuous function on the model output. Assume $J(O(a, \theta))$ signifies a continuous objective function $f: \Bbb R^n \to \Bbb R$ where $\theta$ is a single-layer (no hidden elements) fully connected neural network with continuous activation. We define a `stable' output with respect to the model's parameters in Equation (\ref{eq11}), where $c$ is some tensor of dimension matching $\theta$ composed of constants and $\delta$ is some tensor of sufficiently small constants.  The larger the norm of $c$ for some fixed $\delta$, the more stable the output with respect to the model's parameters.
    
    \begin{equation}
        J(O(a, \theta)) - J(O(a, \theta + c)) < \delta
        \label{eq11}
    \end{equation}
    
    Consider the following: there are two arguments that determine the value of $O(a, \theta)$ and therefore by a change of variables we can convert a small change in $\theta$ while keeping $a$ constant into a change in $a$ while keeping $\theta$ constant (assuming that $O$ and $J$ are continuous functions). 
    
    Specifically, the case of adding a constant scalar $x$ to some input element $e \in a$ is equivalent to adding $c/l$ to all weights from that input element $e$ to the first layer of that model, where $l$ is the number of elements in the first layer.  There is in this case an (not necessarily unique) input $a'$ such that $O(a, \theta') = O(a', \theta)$.
    
    Given the transformation $f(Wa + b) = Y$ on the input $a$ given bias vector $b$ and an invertible weight matrix $W$ and some continuous nonlinear invertible function $f$, the value of an input necessary to match some $Y$ is given in (\ref{eq12})
    
    \begin{equation}
        a = W^{-1}(f^{-1}(Y) - b)
        \label{eq12}
    \end{equation}
    
    If $W$ is non-invertible (which is commonly the case in typical deep learning settings) or a non-unique function such as ReLU is applied to layer activations, then many possible inputs $a_1, a_2, a_3, ...$ can be used to give $Y$.  In either case, we find a value of some tensor $d$ that satisfies Equation (\ref{eq13}).
    
    \begin{equation}
        \forall \theta, \; \exists d : J(O(a, \theta + c)) = J(O(a + d, \theta))
        \label{eq13}
    \end{equation}
    
    For a stable output $O(a, \theta)$ with respect to parameters $\theta$ while keeping the input $a$ constant, that output is also stable to changes in the input $a$ while keeping the model parameters $\theta$ constant as the inequality in Equation \ref{eq14} holds. 
    
    \begin{equation}
        J(O(a, \theta)) - J(O(a, \theta + c)) = J(O(a, \theta)) - J(O(a + d, \theta)) < \delta
        \label{eq14}
    \end{equation}
    
    Therefore stability of the output's loss with respect to the model's parameters for some dataset is equivalent to stability of the output's loss with respect to the model's input given fixed parameters, meaning the given definition of model generalization has been fulfilled. Furthermore, the larger the $c$ -region around $\theta$ such that (\ref{eq11}) holds the larger the $d$-region around $a$ such that the same inequality is observed, assuming that the linear transformation $O(a, \theta)$ has no zero eigenvalues.

    The special case of a one-hidden-layer fully connected architecture is clearly extendable to the more general case of neural networks without fully connected layers (because such layers are usually equivalent to a subset of fully connected layers) as well as models with many sequential layers simply by inverting each layer in succession and finding the input or inputs that yield some identical output. 
    \end{proof}
    
    The definition of generalization used above leaves something to be desired, and in particular it does not address cases in which test set members are not `close' to any given training input as measured by a metric such as $L^2$, but are `close' in that they have some abstract invariance. We sketch an informal proof that stable model configurations $\theta$ may be shown to generalize according to this definition as well, which the reader can clearly see is analagous the the proof above and may be formalized in a similar manner.
    
    Assume that some model maps inputs to a low-dimensional manifold and that this mapping is contractive (both assumptions having experimental evidence for vision models currently in use \citep{badger2022}) and that the model output $O(a, \theta)$ arrived at by gradient descent is stable with respect to the model's parameters. Then some small shift in the model's parameters $c$ given some input $a$ corresponds to a small shift in the output of that model (denoted $b$) such that $||O(a, \theta + c) - (O(a, \theta) + b) || < \delta$. The input $a$ is mapped to some low-dimensional manifold in the output space such that the shifted output is in the $\delta$-neighborhood of that manifold. As the model's mapping is assumed to be contractive in the forward direction, a small shift in the output yields a much larger shift in the input, or in other words $d$ such that $|| O(a + d, \theta) - O(a, \theta) || < \epsilon $ implies $|| a + d - a || > \epsilon$ for some small constant $\epsilon$, which fulfills informally the required lack of `closeness' in the input.  But as the output is still near the manifold, the corresponding input still yields low loss assuming that the manifold is accurate.  This manifold itself is the abstract invariance that determines `closeness' between inputs, and therefore stable model configurations also generalize with respect to larger changes on the input as long as these are approximately mapped to the correct manifold.
    
    The former generalization definition (assuming only smoothness) implies approximately-continuous transformations of the input whereas the latter implies discrete transformations. If a few simplifying assumptions on a model are made, it can furthermore be said that small changes in early layer parameters are equivalent to near-continuous shifts in the input whereas small changes in deep layer parameters correspond to discrete input shifts.
    
    Thus the points of the model space that are attractive to $\theta$ during gradient descent are exactly those that are general, at least when the transformations given by $a + d$ corresponding to allowable $\theta + c$ that exist in the training set also exist in the test set. The equivalency between changes in $\theta$ and changes in $a$ or $O(a, \theta)$ are determined by the model architecture, with more expressive models typical of deep learning approaches being capable of equating not just changes in pixel values (as in the example above) but also abstract objects.  The idea that a function exists which maps large changes in $a$ (as measured by a metric such as $L^2$) to certain small changes in some low-dimensional subspace of $O(a, \theta)$ is equivalent to the manifold learning hypothesis \citep{bengio2012unsupervised}.

    For the discrete definition of stability, it is worth noting the connection between the subset of possible values of $c$ in Equation (\ref{eq14}) and the possible values of $d$ that are present in the dataset: it is valid to view $d$ as differences between images of one class that are present in the training dataset for which the output $O(a + d, \theta)$ depends.  Therefore the $c$ that $\theta$ learns to become insensitive to is a representation of the intra-class differences for the model as observed by the training set.  In the case where the test set does not resemble the training set (as for fixed random inputs), $O(a, \theta + c)$ is not stable for changes in $d$ that are typical of those necessary to transform $a_{train} \to a_{test}$, or different examples of $a_{test} \to a_{test}$.
    
    This is important to understanding the presence of adversarial negatives which at first glance seem to contradict Theorem 1.  But upon close inspection they do not because the only transformation represented by $d$ that the model becomes insensitive to are those learned by the dataset, or in other words the same transformations that exist between examples $a$ of each class. Adversarial examples are of extremely low probability in the input space and therefore are by definition poor examples of the transformations $d$ between input examples.
    
\section{Implications}

    In summary, we provide theory on why deep learning tends to avoid memorization: when trained via gradient descent, models head towards attractive points in parameter space that are more stable with respect to those parameters, which is equivalent to attractive points in input space given fixed parameters.  
    
    Our theoretical exploration also provides explanations for a number of somewhat confusing observations with respect to deep learning generality.  The observation that increasing a deep model's size does not increase overfitting for a sufficiently large dataset \citep{neyshabur2014} is to be expected given than a large model would only be capable of approximating a greater number of possible transformations on the input compared to a small one, but gradient descent would still be expected to result in a generalizable configuration of parameters.  The theory also explains why fitting noise takes potentially fewer epochs than fitting real data: in the absence of some faraway (compared to the initial configuration $\theta_0$) strong attractor, weak and relatively unstable nearby attractors are instead found. Additional evidence for this idea comes when one observes the lack of stability of model configurations that overfit. Overfitting occurs when either no sufficiently strong attractor can be found, or more commonly when the transformations on the training set inputs that are stable for some model's configurations do not correspond to the changes between training and test set which is equivalently the case if the manifold learned during training does not accurately match the manifold of the test data.
    
    We find support for the notion that generalization is understood in the context of the dataset a given model is applied to \citep{arpit2017closer} being that gradient descent-based deep learning has the potential to generalize implicitly, given sufficient capacity to form an accurate embedding.

    Implicit regularization from gradient descent does not preclude the notion that certain deep learning architectures are more or less susceptible to memorization.  The finding that transformer architectures appear to be capable of slightly more generalization after extensive pre-training than similar convolutional models \citep{dosovitskiy2020image} suggests that there are indeed differences in the ability of various models to generalize, but in general it remains unclear why this would be.
    
    The methods used in this work to measure the ease of overfitting may also be applied to estimate a new dataset's quality.  For a given model if there is an insufficiently large gap between the number of epochs required for learning via (\ref{eq1}) versus overfitting (\ref{eq2}) then the dataset may be prone to memorization when trained on various deep learning models.  

\section{Conclusion}

    Gradient descent-based learning procedures are at the heart of most of deep learning today as they allow for the optimization of intractable objective functions. Important to this approach's success when applied to very large models capable of memorizing the training data is the globally suboptimal nature of gradient updates with respect to minimizing the chosen loss function.  This lack of global optimization is observed to be paired with an empirical difficulty of finding model parameters that severely overfit training data, assuming an approximation of the correct manifold describing the data is found by the model. Our work gives an explanation as to why attractive points of $\theta$ generalize when the test data contains the same invariants as the training data. We present a perspective in which gradient descent, though accurately described as a greedy and generally non-optimal algorithm with respect to global objective function minimization, is in fact (assuming random initialization) globally optimal with respect to the stability of the eventual model configuration upon convergence.  
    
\bibliographystyle{unsrtnat}
\bibliography{references}  

\beginsupplement
\section{Appendix}

    Code for this work may be found at \url{https://github.com/blbadger/learning-generality}. Experiments were performed on Colab, and models are trained using PyTorch's version of Cross-Entropy, such that the output of the model is log-Softmax transformed before the negative log-likelihood is computed.  As a Softmax layer exists in the models in this work, this procedure effectively Softmax-transforms twice in succession. Learning rates are all 1e-4 except where otherwise noted (Figures \ref{fig1} (a), \ref{figs1}).  
    
    For the CIFAR10 datasets, the first 2,560 test and train samples are termed the `2.5k' dataset, whereas the first 5120 samples are termed the `5k' and the first 10,000 test and 10,240 training samples are referred to as the `10k' dataset.  For Figure \ref{figs1}, test accuracy is subtracted from training accuracy due to the tendency of ResNet18 to experience an initial increase in test accuracy when (\ref{eq4}) is applied.
    
    \begin{figure}[h]
        \centering
        \includegraphics[width=0.95\textwidth]{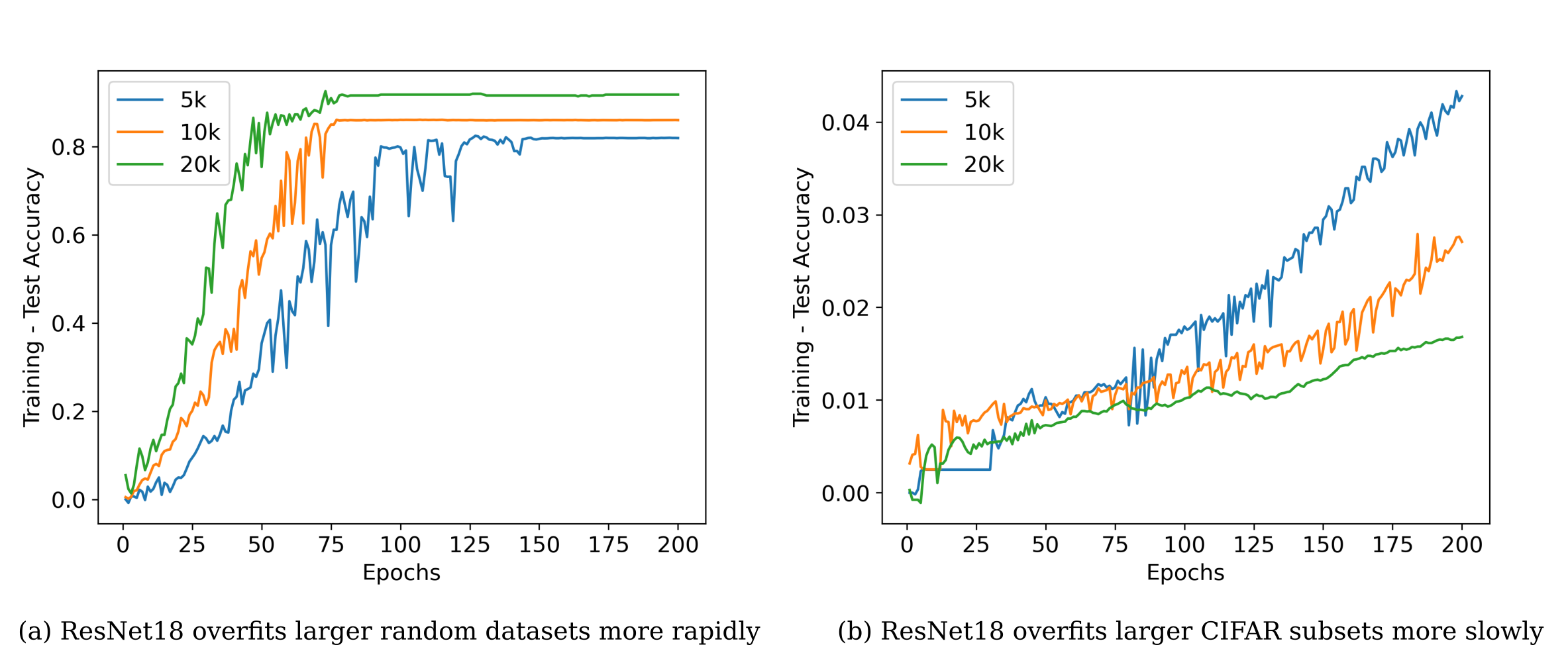}
        \caption{Dataset size increases has opposite effects for the ease of overfitting for noise versus CIFAR10 subsets with ResNet18 applied to (\ref{eq4}), $k=0.11$. Note that CIFAR10 `test' and `training' sets are equal-sized subsets of the training set. (a) Overfitting curves (training accuracy - test accuracy) for random Gaussian datasets of various sizes, learning rate of 1e-4 and no Softmax added. (b) Overfitting curves for CIFAR10 subsets, lr=5e-6 with a Softmax transformation on the output.}
        \label{figs1}
    \end{figure}
    
\end{document}